\documentclass{ecai2014}
\usepackage{times}
\usepackage{graphicx}
\usepackage{float}
\usepackage{tikz}
\usepackage{latexsym,verbatim} 
\usepackage{enumitem}
\usepackage{amsmath}
\usepackage{amsthm}
\usepackage{pict2e}
\usepackage{amssymb}
\usetikzlibrary{arrows,automata}
\usepackage{url}
\usepackage{color}

\newcommand{\fw}{F}
\newcommand{\fwb}{G}
\newcommand{\allf}{\mathcal{F}}
\newcommand{\args}{A}
\newcommand{\uargs}{\mathcal{U}}

\newcommand{\paf}{\textit{M}}
\newcommand{\alp}{\textit{ALP}}
\newcommand{\aafs}{I}

\newcommand{\attrel}{\rightsquigarrow}

\newcommand{\omoveattack}[2]{\ensuremath{\textbf{OPP:}\hspace{.2em}#1\attrel #2}}

\newcommand{\hpmoveattack}[2]{\ensuremath{\textbf{PRO:}\hspace{.2em}#1 \attrel^{+} #2}}
\newcommand{\hpmoveattackneg}[2]{\ensuremath{\textbf{PRO:}\hspace{.2em}#1 \attrel^{-} #2}}

\newcommand{\omoveok}{\ensuremath{\textbf{OPP:\hspace{.2em}ok}}}
\newcommand{\pmovewin}{\ensuremath{\textbf{PRO:\hspace{.2em}win}}}

\newtheorem{definition}{Definition}
\newtheorem{thm}{Theorem}
\newtheorem{prp}{Proposition}

\newtheorem{lemm}{Lemma}
\newtheorem{example}{Example}
\newtheorem{remark}{Remark}

\newcommand{\negf}{\mathord{\sim}}

\begin{document}

\title{Abduction and Dialogical Proof \\ in Argumentation and Logic Programming}

\author{Richard Booth\footnote{\hspace{-4pt}Computer Science and Communication, University of Luxembourg (richard.booth@uni.lu, tjitze.rienstra@uni.lu, leon.vandertorre@uni.lu)} \and \hspace{-6pt} Dov Gabbay$^{1,}$\footnote{Dept. Computer Science, King's College London (dov.gabbay@kcl.ac.uk)} \and \hspace{-6pt} Souhila Kaci\footnote{LIRMM, University of Montpellier 2 (souhila.kaci@lirmm.fr)} \and \hspace{-4pt}Tjitze Rienstra$^{1,3}$ \hspace{-7pt} \and \hspace{-4pt}Leendert van der Torre$^{1}$}

\maketitle
\bibliographystyle{ecai2014}

\begin{abstract}
We develop a model of abduction in abstract argumentation, where changes to an argumentation framework act as hypotheses to explain the support of an observation.
We present dialogical proof theories for the main decision problems (i.e., finding hypotheses that explain skeptical/credulous support) and we show that our model can be instantiated on the basis of abductive logic programs.
\end{abstract}

\section{Introduction}

In the context of abstract argumentation~\cite{DBLP:journals/ai/Dung95}, abduction can be seen as the problem of finding \emph{changes} to an argumentation framework (or \emph{AF} for short) with the goal of explaining observations that can be justified by making arguments accepted.
The general problem of whether and how an AF can be changed with the goal of changing the status of arguments has been studied by Baumann and Brewka~\cite{DBLP:conf/comma/BaumannB10}, who called it the \emph{enforcing} problem, as well as Bisquert et al.~\cite{DBLP:conf/sum/BisquertCSL13}, Perotti et al.~\cite{DBLP:conf/tafa/BoellaGPTV11} and Kontarinis et al.~\cite{kontarinis2013rewriting}. 
None of these works, however, made any explicit link with abduction.
Sakama~\cite{sakama}, on the other hand, explicitly focused on abduction, and presented a model in which additions as well as removals of arguments from an abstract AF act as explanations for the observation that an argument is accepted or rejected.

While Sakama did address computation in his framework, his method was based on translating abstract AFs into logic programs.
Proof theories in argumentation are, however, often formulated as \emph{dialogical} proof theories, which aim at relating the problem they address with stereotypical patterns found in real world dialogue.
For example, proof theories for skeptical/credulous acceptance have been modelled as dialogues in which a proponent persuades an opponent to accept the necessity/possibility of an argument~\cite{Modgil2009}, while credulous acceptance has also been related to Socratic style dialogue~\cite{caminada2010preferred}.
Thus, the question of how decision problems in abduction in argumentation can similarly be modelled as dialogues remains open.

Furthermore, argumentation is often used as an abstract model for non-monotonic reasoning formalisms.
For example, an \emph{instantiated} AF can be generated on the basis of a logic program.
Consequences can then be computed by looking at the extensions of the instantiated AF~\cite{DBLP:journals/ai/Dung95}.
In the context of abduction, one may ask whether a model of abduction in argumentation can similarly be seen as an abstraction of \emph{abductive} logic programming.
Sakama, however, did not explore the instantiation of his model, meaning that this question too remains open.

This brings us to the contribution of this paper.
We first present a model of abduction in abstract argumentation, based on the notion of an AAF (abductive argumentation framework) that encodes different possible changes to an AF, each of which may act as a hypothesis to explain an observation that can be justified by making an argument accepted.
We then do two things:

\begin{enumerate}
\item We present sound and complete dialogical proof procedures for the main decision problems, i.e., finding hypotheses that explain skeptical/credulous acceptance of arguments in support of an observation.
These proof procedures show that the problem of abduction is related to an extended form of persuasion, where the proponent 
uses \emph{hypothetical} moves to persuade the opponent.

\item We show that AAFs can be instantiated by ALPs (abductive logic programs) in such a way that the hypotheses generated for an observation by the ALP can be computed by translating the ALP into an AAF. 
The type of ALPs we focus on are based on Sakama and Inoue's model of \emph{extended} abduction~\cite{DBLP:conf/ijcai/InoueS95,DBLP:journals/amai/InoueS99}, in which hypotheses have a positive as well as a negative element (i.e., facts added to the logic program as well as facts removed from it).
\end{enumerate}

In sum, our contribution is a model of abduction in argumentation with dialogical proof theories for the main decision problems, which can be seen as an abstraction of abduction in logic programming.

The overview of this paper is as follows.
After introducing the necessary preliminaries in section~\ref{sec:preliminaries} we present in section~\ref{sec:preferentialafs} our model of abduction in argumentation.
In section~\ref{sec:explanationdialogues} we present dialogical proof procedures for the main decision problems (explaining skeptical/credulous acceptance).
In section~\ref{sec:abductioninlp} we show that our model of abduction can be used to instantiate abduction in logic programming.
We discuss related work in section~\ref{sec:relatedwork} and conclude in section~\ref{sec:conclusion}.

\section{Preliminaries}\label{sec:preliminaries}

An argumentation framework consists of a set $\args$ of arguments and a binary \emph{attack} relation $\attrel$ over $\args$~\cite{DBLP:journals/ai/Dung95}.
We assume in this paper that $\args$ is a finite subset of a fixed set $\uargs$ called the \emph{universe of arguments}.

\begin{definition}
Given a countably infinite set $\uargs$ called the \emph{universe of arguments}, an \emph{argumentation framework} (\emph{AF}, for short) is a pair $\fw = (\args, \attrel)$ where $\args$ is a finite subset of $\uargs$ and $\attrel$ a binary relation over $\args$.
If $a \attrel b$ we say that $a$ \emph{attacks} $b$.
$\allf$ denotes the set of all AFs.
\end{definition}

\emph{Extensions} are sets of arguments that represent different viewpoints on the acceptance of the arguments of an AF.
A \emph{semantics} is a method to select extensions that qualify as somehow justifiable.
We focus on one of the most basic ones, namely the \emph{complete} semantics~\cite{DBLP:journals/ai/Dung95}. 

\begin{definition}\label{defn:semantics}
Let $\fw = (\args, \attrel)$.
An \emph{extension} of $\fw$ is a set $E \subseteq \args$.
An extension $E$ is \emph{conflict-free} iff for no $a, b \in E$ it holds that $a \attrel b$.
An argument $a \in \args$ is \emph{defended} by $E$ iff for all $b$ such that $b \attrel a$ there is a $c \in E$ such that $c \attrel b$.
Given an extension $E$, we define $\textit{Def}_{\fw}(E)$ by $\textit{Def}_{\fw}(E) = \{ a \in \args \mid E\mbox{ defends }a\}$.
An extension $E$ is \emph{admissible} iff $E$ is conflict-free and $E \subseteq \textit{Def}_{\fw}(E)$, and \emph{complete} iff $E$ is conflict-free and $E = \textit{Def}_{\fw}(E)$.
The set of complete extension of $\fw$ will be denoted by $Co(\fw)$.
Furthermore, the grounded extension (denoted by $Gr(\fw)$) is the unique minimal (w.r.t. $\subseteq$) complete extension of $\fw$.
\end{definition}

An argument is said to be \emph{skeptically} (resp. \emph{credulously}) accepted iff it is a member of all (resp. some) complete extensions.
Note that the set of skeptically accepted arguments coincides with the grounded extension.
Furthermore, an argument is a member of a complete extension iff it is a member of a \emph{preferred} extension, which is a maximal (w.r.t. $\subseteq$) complete extension.
Consequently, credulous acceptance under the preferred semantics (as studied e.g. in~\cite{Modgil2009}) coincides with credulous acceptance under the complete semantics.

\section{Abductive AFs}\label{sec:preferentialafs}

Abduction is a form of reasoning that goes from an observation to a hypothesis.
We assume that an observation translates into a set $X \subseteq \args$.
Intuitively, $X$ is a set of arguments that each individually support the observation.
If at least one argument $x \in X$ is skeptically (resp. credulously) accepted, we say that the observation $X$ is skeptically (resp. credulously) \emph{supported}.

\begin{definition}
Given an AF $\fw = (\args, \attrel)$, an observation $X \subseteq \args$ is skeptically (resp. credulously) supported iff for all (resp. some) $E \in Co(\fw)$ it holds that $x \in E$ for some $x \in X$.
\end{definition}

The following proposition implies that checking whether an observation $X$ is skeptically supported can be done by checking whether an individual argument $x \in X$ is in the grounded extension.

\begin{prp}\label{prp:individualobservation}
Let $\fw = (\args, \attrel)$ and $X \subseteq \args$.
It holds that $\fw$ skeptically supports $X$ iff $x \in Gr(\fw)$ for some $x \in X$.
\end{prp}

\begin{proof}[Proof of proposition~\ref{prp:individualobservation}]
The if direction is immediate.
For the only if direction, assume $\fw = (\args, \attrel)$ explains skeptical support for $X$.
Then for every complete extension $E$ of $\fw$, there is an $x \in X$ s.t. $x \in E$.
Define $\fwb$ by $\fwb = (\args \cup \{a, b\}, \attrel \cup \{ (x, a) \mid x \in X \} \cup \{ (a, b) \})$, where $a, b \not \in \args$.
Then for every complete extension $E$ of $\fwb$ it holds that $b \in E$, hence $b \in Gr(\fwb)$.
Thus $x \in Gr(\fwb)$ for some $x \in X$.
But $Gr(\fw) = Gr(\fwb) \cap \args$, hence $x \in Gr(\fw)$ for some $x \in X$.
\end{proof}

It may be that an AF $\fw$ does not skeptically or credulously support an observation $X$.
Abduction then amounts to finding a change to $\fw$ so that $X$ is supported.
We use the following definition of an \emph{AAF} (\emph{Abductive AF}) to capture the changes w.r.t. $\fw$ (each change represented by an AF $\fwb$ called an \emph{abducible} AF) that an agent considers.
We assume that $\fw$ itself is also an abducible AF, namely one that captures the case where no change is necessary.
Other abducible AFs may be formed by addition of arguments and attacks to $\fw$, removal of arguments and attacks from $\fw$, or a combination of both.

\begin{definition}\label{defn:preferentialaf}
An \emph{abductive AF} is a pair \mbox{$\paf = (\fw, \aafs)$} where $\fw$ is an AF and $\aafs \subseteq \allf$ a set of AFs called \emph{abducible} such that $\fw \in \aafs$.
\end{definition}

Given an AAF $(\fw, \aafs)$ and observation $X$, skeptical/credulous support for $X$ can be explained by the change from $\fw$ to some $\fwb \in \aafs$ that skeptically/credulously supports $X$.
In this case we say that $\fwb$ \emph{explains} skeptical/credulous support for $X$.
The arguments/attacks added to and absent from $\fwb$ can be seen as the actual explanation.

\begin{definition}
Let $\paf = (\fw, \aafs)$ be an AAF.
An abducible AF $\fwb \in \aafs$ \emph{explains} skeptical (resp. credulous) support for an observation $X$ iff $\fwb$ skeptically (resp. credulously) supports $X$.
\end{definition}

One can focus on explanations satisfying additional criteria, such as minimality w.r.t. the added or removed arguments/attacks.
We leave the formal treatment of such criteria for future work.

\begin{example}\label{examp:abd}
Let $\paf = (\fw, \{\fw, \fwb_1, \fwb_2, \fwb_3\})$, where $\fw, \fwb_1, \fwb_2$ and $\fwb_3$ are as defined in figure~\ref{fig:mfws}.
Let $X = \{b\}$ be an observation.
It holds that $\fwb_1$ and $\fwb_3$ both explain skeptical support for $X$, while $\fwb_2$ only explains credulous support for $X$.
\begin{figure}
\centering
\begin{tikzpicture}
 [scale=.7]
  \node (l1) at 	(0,-2) {$\fw$};
  \node (rb) at 	(0,0) [circle, draw, minimum size=11pt, inner sep=0pt] {b};
  \node (rc) at 	(1,0) [circle, draw, minimum size=11pt, inner sep=0pt] {c};
  \node (ra) at 	(-1,0) [circle, draw, minimum size=11pt, inner sep=0pt] {a};
  \node (rd) at 	(-1,1) [circle, draw, minimum size=11pt, inner sep=0pt] {d};
  \foreach \from/\to in {ra/rb,rb/rc,rc/rb}
    \draw [->] (\from) -- (\to);

  \node (l2) at 	(3,-2) {$\fwb_1$};
  \node (sb) at 	(3,0) [circle, draw, minimum size=11pt, inner sep=0pt] {b};
  \node (sc) at 	(4,0) [circle, draw, minimum size=11pt, inner sep=0pt] {c};
  \node (sa) at 	(2,0) [circle, draw, minimum size=11pt, inner sep=0pt] {a};
  \node (sd) at 	(2,1) [circle, draw, minimum size=11pt, inner sep=0pt] {d};
  \node (se) at 	(3,-1) [circle, draw, minimum size=11pt, inner sep=0pt] {e};
  \foreach \from/\to in {sa/sb,sb/sc,sc/sb,se/sa,se/sc}
    \draw [->] (\from) -- (\to);

  \node (l3) at 	(6,-2) {$\fwb_2$};
  \node (tb) at 	(6,0) [circle, draw, minimum size=11pt, inner sep=0pt] {b};
  \node (tc) at 	(7,0) [circle, draw, minimum size=11pt, inner sep=0pt] {c};
  \foreach \from/\to in {tb/tc,tc/tb}
    \draw [->] (\from) -- (\to);

  \node (l4) at 	(9,-2) {$\fwb_3$};
  \node (ub) at 	(9,0) [circle, draw, minimum size=11pt, inner sep=0pt] {b};
  \node (uc) at 	(10,0) [circle, draw, minimum size=11pt, inner sep=0pt] {c};
  \node (ue) at 	(9,-1) [circle, draw, minimum size=11pt, inner sep=0pt] {e};
  \foreach \from/\to in {ub/uc,uc/ub,ue/uc}
    \draw [->] (\from) -- (\to);
\end{tikzpicture}
\caption{The AFs of the AAF $(\fw, \{\fw, \fwb_1, \fwb_2, \fwb_3\})$.}
\label{fig:mfws}
\end{figure}
\end{example}

\vspace{-30pt}\begin{remark}\label{remark1}
The main difference between Sakama's~\cite{sakama} model of abduction in abstract argumentation and the one presented here, is that he takes an explanation to be a set of independently selectable abducible arguments, while we take it to be a change to the AF that is applied as a whole.
In section~\ref{sec:abductioninlp} we show that this is necessary when applying the abstract model in an instantiated setting.
\end{remark}

\section{Explanation dialogues}\label{sec:explanationdialogues}

In this section we present methods to determine, \emph{given an AAF $\paf = (\fw, \aafs)$ (for  $\fw = (\args, \attrel)$) whether an abducible AF $\fwb \in \aafs$ explains credulous or skeptical support for an observation $X \subseteq \args$}.
We build on ideas behind the \emph{grounded} and \emph{preferred games}, which are dialogical procedures that determine skeptical or credulous acceptance of an argument~\cite{Modgil2009}.
To sketch the idea behind these games (for a detailed discussion cf.~\cite{Modgil2009}): two imaginary players (PRO and OPP) take alternating turns in putting forward arguments according to a set of rules,
	PRO either as an initial claim or in defence against OPP's attacks, while OPP initiates different disputes by attacking the arguments put forward by PRO.
Skeptical or credulous acceptance is proven if PRO can win the game by ending every dispute in its favour according to a ``last-word'' principle.

Our method adapts this idea so that the moves made by PRO are essentially \emph{hypothetical} moves.
That is, to defend the initial claim (i.e., to explain an observation) PRO can put forward, by way of hypothesis, any attack $x \attrel y$ present in some $\fwb \in \aafs$.
This marks a choice of PRO to focus only on those abducible AFs in which the attack $x \attrel y$ is present.
Similarly, PRO can reply to an attack $x \attrel y$, put forward by OPP, with the claim that this attack is invalid, marking the choice of PRO to focus only on the abducible AFs in which the attack $x \attrel y$ is \emph{not} present.
Thus, each move by PRO narrows down the set of abducible AFs in which all of PRO's moves are valid.
The objective is to end the dialogue with a non-empty set of abducible AFs.
Such a dialogue represents a proof that these abducible AFs explain skeptical or credulous support for the observation.

Alternatively, such dialogues can be seen as games that determine skeptical/credulous support of an observation by an AF that are played simultaneously over all abducible AFs in the AAF.
In this view, the objective is to end the dialogue in such a way that it represents a proof for at least one abducible AF.
Indeed, in the case where $\paf = (\fw, \{\fw\})$, our method reduces simply to a proof theory for skeptical or credulous support of an observation by $\fw$.

Before we move on we need to introduce some notation.

\begin{definition}
Given a set $\aafs$ of AFs we define:
\vspace{-8pt}
\begin{itemize}
\item $\args_{\aafs} = \cup \{\args \mid (\args, \attrel) \in \aafs \}$,
\item $\attrel_{\aafs} = \cup \{ \attrel \mid (\args, \attrel) \in \aafs \}$,
\item $\aafs_{x \attrel y} = \{ (\args, \attrel) \in \aafs \mid x, y \in \args, x \attrel y\}$,
\item $\aafs_{X} = \{ (\args, \attrel) \in \aafs \mid X \subseteq \args\}$.
\end{itemize}
\end{definition}

We model dialogues as sequences of \emph{moves}, each move being of a certain type, and made either by PRO or OPP.

\begin{definition}
Let $\paf = (\fw, \aafs)$ be an AAF.
A \emph{dialogue based on $\paf$} is a sequence $S = (m_1, \ldots, m_n)$, where each $m_i$ is either:
\vspace{-8pt}
\begin{itemize}
\item an \emph{OPP attack} ``$\omoveattack{x}{y}$'', where $x \attrel_\aafs y$,
\item a \emph{hypothetical PRO defence} ``$\hpmoveattack{y}{x}$'', where $y \attrel_\aafs x$,
\item a \emph{hypothetical PRO negation} ``$\hpmoveattackneg{y}{x}$'', where $y \attrel_\aafs x$,
\item a \emph{conceding move} ``$\omoveok$'',
\item a \emph{success claim move} ``$\pmovewin$''.
\end{itemize}
We denote by $S \cdot S'$ the concatenation of $S$ and $S'$.
\end{definition}

Intuitively, a move $\omoveattack{y}{x}$ represents an attack by OPP on the argument $x$ by putting forward the attacker $y$.
A hypothetical PRO defence $\hpmoveattack{y}{x}$ represents a defence by PRO who puts forward $y$ to attack the argument $x$ put forward by OPP.
A hypothetical PRO negation $\hpmoveattackneg{y}{x}$, on the other hand, represents a claim by PRO that the attack $y \attrel x$ is \emph{not} a valid attack.
The conceding move $\omoveok$ is made whenever OPP runs out of possibilities to attack a given argument, while the move $\pmovewin$ is made when PRO is able to claim success.

In the following sections we specify how dialogues are structured.
Before doing so, we introduce some notation that we use to keep track of the abducible AFs on which PRO chooses to focus in a dialogue $D$.
We call this set the \emph{information state} of $D$ after a given move.
While it initially contains all abducible AFs in $\paf$, it is restricted when PRO makes a move $\hpmoveattack{x}{y}$ or $\hpmoveattackneg{x}{y}$.

\begin{definition} 
Let $\paf = (\fw, \aafs)$ be an AAF.
Let $D = (m_1, \ldots, m_n)$ be a dialogue based on $\paf$.
We denote the \emph{information state in $D$ after move $i$} by $J(D, i)$, which is defined recursively by:

$J(D, i) =
\begin{cases} 
\aafs											&	\mbox{if } i = 0, \\
J(D, i-1) \cap \aafs_{x \attrel y}						&	\mbox{if } m_{i} = \hpmoveattack{x}{y}, \\
J(D, i-1) \setminus \aafs_{x \attrel y}					&	\mbox{if } m_{i} = \hpmoveattackneg{x}{y}, \\
J(D, i-1)											&	\mbox{otherwise.}
\end{cases}
$

We denote by $J(D)$ the information state $J(D, n)$.

\end{definition}

\subsection{Skeptical explanation dialogues}

We define the rules of a dialogue using a set of production rules that recursively define the set of sequences constituting dialogues.
(The same methodology was used by Booth et al.~\cite{DBLP:conf/aldt/BoothKR13} in defining a dialogical proof theory related to preference-based argumentation.)
In a skeptical explanation dialogue for an observation $X$, an initial argument $x \in X$ is challenged by the opponent, who puts forward all possible attacks $\omoveattack{y}{x}$ present in any of the abducible AFs present in the AAF, followed by $\omoveok$.
We call this a \emph{skeptical OPP reply} to $x$.
For each move $\omoveattack{y}{x}$, PRO responds with a \emph{skeptical PRO reply to $y \attrel x$}, which is either a hypothetical defence $\hpmoveattack{z}{y}$ (in turn followed by a skeptical OPP reply to $z$) or a hypothetical negation $\hpmoveattackneg{y}{x}$.
Formally:

\begin{definition}[Skeptical explanation dialogue]
Let $\fw = (\args, \attrel)$, $\paf = (\fw, \aafs)$ and $x \in \args$.
\vspace{-6pt}
\begin{itemize}
\item A \emph{skeptical OPP reply to $x$} is a finite sequence 
	$(\omoveattack{y_1}{x}) \cdot S_1 \cdot \ldots \cdot (\omoveattack{y_n}{x}) \cdot S_n \cdot (\omoveok)$ 
		where $\{y_1, \ldots, y_n\} = \{ y \mid y \attrel_\aafs x\}$ 
		and each $S_i$ is a skeptical PRO reply to $y_i \attrel x$.
\item A \emph{skeptical PRO reply to $y \attrel x$} is either: 
(1) A sequence $(\hpmoveattack{z}{y}) \cdot S$ where $z \attrel_\aafs y$ and where $S$ is a skeptical OPP reply to $z$, or
(2) The sequence $(\hpmoveattackneg{y}{x}).$
\end{itemize}
Given an observation $X \subseteq \args$ we say that $\paf$ generates the \emph{skeptical explanation dialogue} $D$ for $X$ iff $D = S \cdot (\pmovewin)$, where $S$ is a skeptical OPP reply to some $x \in X$.
\end{definition}

The following theorem establishes soundness and completeness.

\begin{thm}\label{thm:scskeptical}
Let $\paf = (\fw, \aafs)$ be an AAF where $\fw = (\args, \attrel)$.
Let $X \subseteq \args$ and $\fwb \in \aafs$.
It holds that 
	$\fwb$ explains skeptical support for $X$
	iff
	$\paf$ generates a skeptical explanation dialogue $D$ for $X$ such that $\fwb \in J(D)$.
\end{thm}

Due to space constraints we only provide a sketch of the proof. 

\begin{proof}[Sketch of proof]
Let $\paf = ((\args, \attrel), \aafs)$, $X \subseteq \args$ and $\fwb \in \aafs$.
\textit{(Only if:)} 
Assume $x \in Gr(\fwb)$ for some $x \in X$.
By induction on the number of times the characteristic function~\cite{DBLP:journals/ai/Dung95} is applied so as to establish that $x \in Gr(\fwb)$, it can be shown that a credulous OPP reply $D$ to $x$ exists (and hence a dialogue $D \cdot (\pmovewin)$ for $X$) s.t. $\fwb \in J(D \cdot (\pmovewin))$. 
\textit{(If:)} 
Assume $\paf$ generates a skeptical explanation dialogue $D$ for $X$ s.t. $\fwb \in J(D)$.
By induction on the structure of $D$ it can be shown that $x \in Gr(\fwb)$ for some $x \in X$.
\end{proof}

\begin{example}\label{examp:skeptical}
The listing below shows a skeptical explanation dialogue $D = (m_1, \ldots, m_8)$ for the observation $\{b\}$ that is generated by the AAF defined in example~\ref{examp:abd}.

\begin{tabular}{l | l | l}
$i$			&	$m_i$									&	$J(D, i)$				\\
\hline
1		&	\hspace{0pt}	$\omoveattack{c}{b}$			&	$\{F, G_1, G_2, G_3\}$	\\
2		&	\hspace{10pt}		$\hpmoveattack{e}{c}$		&	$\{G_1, G_3\}$	\\
3		&	\hspace{10pt}		$\omoveok$				&	$\{G_1, G_3\}$	\\
4		&	\hspace{0pt}	$\omoveattack{a}{b}$			&	$\{G_1, G_3\}$	\\
5		&	\hspace{10pt}		$\hpmoveattack{e}{a}$		&	$\{G_1\}$	\\
6		&	\hspace{10pt}		$\omoveok$				&	$\{G_1\}$	\\
7		&	\hspace{0pt}	$\omoveok$					&	$\{G_1\}$	\\
8		&	\hspace{0pt}	$\pmovewin$					&	$\{G_1\}$	\\
\end{tabular}

The sequence $(m_1, \ldots, m_7)$ is a skeptical OPP reply to $b$, in which OPP puts forward the two attacks $c \attrel b$ and $a \attrel b$.
PRO defends $b$ from both $c$ and $a$ by putting forward the attacker $e$ (move 2 and 5).
This leads to the focus first on the abducible AFs $G_1, G_3$ (in which the attack $e \attrel c$ exists) and then on $G_1$ (in which the attack $e \attrel a$ exists).
This proves that $G_1$ explains skeptical support for the observation $\{b\}$.
Another dialogue is shown below.

\begin{tabular}{l | l | l}
$i$			&	$m_i$									&	$J(D, i)$				\\
\hline
1		&	\hspace{0pt}	$\omoveattack{c}{b}$			&	$\{F, G_1, G_2, G_3\}$	\\
2		&	\hspace{10pt}		$\hpmoveattack{e}{c}$		&	$\{G_1, G_3\}$	\\
3		&	\hspace{10pt}		$\omoveok$				&	$\{G_1, G_3\}$	\\
4		&	\hspace{0pt}	$\omoveattack{a}{b}$			&	$\{G_1, G_3\}$	\\
5		&	\hspace{10pt}		$\hpmoveattackneg{a}{b}$	&	$\{G_3\}$	\\
6		&	\hspace{0pt}	$\omoveok$					&	$\{G_3\}$	\\
7		&	\hspace{0pt}	$\pmovewin$					&	$\{G_3\}$	\\
\end{tabular}

Here, PRO defends $b$ from $c$ by using the argument $e$, but defends $b$ from $a$ by claiming that the attack $a \attrel b$ is invalid.
This leads to the focus first on the abducible AFs $G_1, G_3$ (in which the attack $e \attrel c$ exists) and then on $G_3$ (in which the attack $a \attrel b$ does not exist).
This dialogue proves that $G_3$ explains skeptical support for $\{b\}$.
\end{example}

\subsection{Credulous explanation dialogues}

The definition of a credulous explanation dialogue is similar to that of a skeptical one.
The difference lies in what constitutes an acceptable defence.
To show that an argument $x$ is skeptically accepted, $x$ must be defended from its attackers by arguments other than $x$ itself.
For credulous acceptance, however, it suffices to show that $x$ is a member of an admissible set, and hence $x$ may be defended from its attackers by any argument, including $x$ itself.
To achieve this we need to keep track of the arguments that are, according to the moves made by PRO, accepted.
Once an argument $x$ is accepted, PRO does not need to defend $x$ again, if this argument is put forward a second time.

Formally a \emph{credulous OPP reply} to $(x, Z)$ (for some $x \in \args_I$ and set $Z \subseteq \args_I$ used to keep track of accepted arguments) consists of all possible attacks $\omoveattack{y}{x}$ on $x$, followed by $\omoveok$ when all attacks have been put forward.
For each move $\omoveattack{y}{x}$, PRO responds either by putting forward a hypothetical defence $\hpmoveattack{z}{y}$ which (this time \emph{only if} $z \not \in Z$) is followed by a credulous OPP reply to $(z, Z \cup \{z\})$, or by putting forward a hypothetical negation $\hpmoveattackneg{y}{x}$.
We call this response a \emph{credulous PRO reply to $(y \attrel x, Z)$}.
A credulous explanation dialogue for a set $X$ consists of a credulous OPP reply to $(x, \{x\})$ for some $x \in X$, followed by a success claim $\pmovewin$.

In addition, arguments put forward by PRO in defence of the observation may not conflict.
Such a conflict occurs when OPP puts forward $\omoveattack{x}{y}$ and $\omoveattack{y}{z}$ (indicating that both $y$ and $z$ are accepted) while PRO does not put forward $\hpmoveattackneg{y}{z}$.
If this situation does not occur we say that the dialogue is \emph{conflict-free}.

\begin{definition}[Credulous explanation dialogue]
Let $\fw = (\args, \attrel)$, $\paf = (\fw, \aafs)$, $x \in \args$ and $Z \subseteq \args$.
\vspace{-8pt}
\begin{itemize}
\item A \emph{credulous OPP reply to $(x, Z)$} is a finite sequence 
	$(\omoveattack{y_1}{x}) \cdot S_1 \cdot \ldots \cdot (\omoveattack{y_n}{x}) \cdot S_n \cdot (\omoveok)$
		where $\{y_1, \ldots, y_n\} = \{ y \mid y \attrel_\aafs x\}$ 
		and each $S_i$ is a credulous PRO reply to $(y_i \attrel x, Z)$.
\item A \emph{credulous PRO reply to $(y \attrel x, Z)$} is either:
	(1) a sequence $(\hpmoveattack{z}{y}) \cdot S$ such that $z \attrel_\aafs y$, $z \not \in Z$ and $S$ is a credulous OPP reply to $(z, Z \cup \{z\})$,
	(2) a sequence $(\hpmoveattack{z}{y})$ such that $z \attrel_\aafs y$ and $z \in Z$, or
	(3) the sequence $(\hpmoveattackneg{y}{x}).$
\end{itemize}
Given a set $X \subseteq \args$ we say that $\paf$ generates the \emph{credulous explanation dialogue} $D$ for $X$ iff $D = S \cdot (\pmovewin)$, where $S$ is a credulous OPP reply to $(x, \{x\})$ for some $x \in X$.
We say that $D$ is \emph{conflict-free} iff for all $x, y, z \in \args_{I}$ it holds that if $D$ contains the moves $\omoveattack{x}{y}$ and $\omoveattack{y}{z}$ then it contains the move $\hpmoveattackneg{y}{z}$.
\end{definition}

The following theorem establishes soundness and completeness.

\begin{thm}\label{thm:sccredulous}
Let $\paf = (\fw, \aafs)$ be an AAF where $\fw = (\args, \attrel)$.
Let $X \subseteq \args$ and $\fwb \in \aafs$.
It holds that 
	$\fwb$ explains credulous support for $X$
	iff
	$\paf$ generates a conflict-free credulous explanation dialogue $D$ for $X$ such that $\fwb \in J(D)$.
\end{thm}

\begin{proof}[Sketch of proof.]
Let $\paf = ((\args, \attrel), \aafs)$, $X \subseteq \args$ and $\fwb \in \aafs$.
\textit{(Only if:)} 
Assume for some $x \in X$ and $E \in Co(\fwb)$ that $x \in E$.
Using the fact that $E \subseteq Def_{\fwb}(E)$ one can recursively define a credulous OPP reply $D$ to $(x, Z)$ for some $Z \subseteq \args$ and hence a credulous explanation dialogue $D \cdot (\pmovewin)$.
Conflict-freeness of $E$ implies conflict-freeness of $D$.
\textit{(If:)} 
Assume $\paf$ generates a credulous explanation dialogue $D \cdot (\pmovewin)$ for $X$ such that $\fwb \in J(D)$.
Then $D$ is a credulous OPP reply to $(a, \{a\})$ for some $a \in X$.
It can be shown that the set $E = \{a\} \cup \{ x \mid \hpmoveattack{x}{z} \in D\}$ satisfies $E \subseteq Def_{\fwb}(E)$.
Conflict-freeness of $D$ implies conflict-freeness of $E$.
Hence $a \in E$ for some $E \in Co(\fwb)$.
\end{proof}

\begin{example}
The listing below shows a conflict-free credulous explanation dialogue $D = (m_1, \ldots, m_6)$ for the observation $\{b\}$  generated by the AAF defined in example~\ref{examp:abd}.

\begin{tabular}{l | l | l}
$i$		&	$m_i$									&	$J(D, i)$					\\
\hline
1		&	\hspace{0pt}	$\omoveattack{c}{b}$			&	$\{\fw, \fwb_1, \fwb_2, \fwb_3\}$\\
2		&	\hspace{10pt}		$\hpmoveattack{b}{c}$		&	$\{\fw, \fwb_1, \fwb_2, \fwb_3\}$\\
3		&	\hspace{0pt}	$\omoveattack{a}{b}$			&	$\{\fw, \fwb_1, \fwb_2, \fwb_3\}$\\
4		&	\hspace{10pt}		$\hpmoveattackneg{a}{b}$	&	$\{\fwb_2, \fwb_3\}$\\
5		&	\hspace{0pt}	$\omoveok$					&	$\{\fwb_2, \fwb_3\}$\\
6		&	\hspace{0pt}	$\pmovewin$					&	$\{\fwb_2, \fwb_3\}$\\
\end{tabular}

Here, the sequence $(m_1, \ldots, m_5)$ is a credulous OPP reply to $(b, \{b\})$.
PRO defends $b$ from OPP's attack $c \attrel b$ by putting forward the attack $b \attrel c$.
Since $b$ was already assumed to be accepted, this suffices.
At move $m_4$, PRO defends itself from the attack $a \attrel b$ by negating it.
This restricts the focus on the abducible AFs $\fwb_2$ and $\fwb_3$.
The dialogue proves that these two abducible AFs explain credulous support for the observation $\{b\}$.
Finally, the skeptical explanation dialogues from example~\ref{examp:skeptical} are also credulous explanation dialogues.
\end{example}

\section{Abduction in logic programming}\label{sec:abductioninlp}

In this section we show that AAFs can be instantiated with abductive logic programs,
	in the same way that regular AFs can be instantiated with regular logic programs.
In sections~\ref{sec:lp} and \ref{sec:lparg} we recall the necessary basics of logic programming and the relevant results regarding logic programming as instantiated argumentation.
In section~\ref{sec:alp} we present a model of abductive logic programming based on Sakama and Inoue's model of extended abduction~\cite{DBLP:conf/ijcai/InoueS95,DBLP:journals/amai/InoueS99}, and in section~\ref{sec:lparg} we show how this model can be instantiated using AAFs.

\subsection{Logic programs and partial stable semantics}\label{sec:lp}

A logic program $P$ is a finite set of rules, each rule being of the form $C \leftarrow A_1, \ldots, A_n, \negf B_1, \ldots, \negf B_m$ where $C, A_1, \ldots, A_n, B_1, \dots, B_m$ are \emph{atoms}.
If $m = 0$ then the rule is called \emph{definite}.
If both $n = 0$ and $m = 0$ then the rule is called a \emph{fact} and we identify it with the atom $C$.
We assume that logic programs are ground.
Alternatively, $P$ can be regarded as the set of ground instances of a set of non-ground rules.
We denote by $At_P$ the set of all (ground) atoms occurring in $P$.
The logic programming semantics we focus on can be defined using \emph{3-valued interpretations}~\cite{DBLP:journals/fuin/Przymusinski90}:

\begin{definition}
A 3-valued interpretation $I$ of a logic program $P$ is a pair $I = (T, F)$ where $T, F \subseteq At_P$ and $T \cap F = \emptyset$.
An atom $A \in At_{P}$ is \emph{true} (resp. \emph{false}, \emph{undecided}) in $I$ iff $A \in T$ (resp. $A \in F$, $A \in At_P \setminus (T \cup F)$).
\end{definition}

The following definition of a \emph{partial stable model} is due to Przymusinski~\cite{DBLP:journals/fuin/Przymusinski90}.
Given a logic program $P$ and 3-valued interpretation $I$ of $P$, the \emph{GL-transformation} $\frac{P}{I}$ is a logic program obtained by
	replacing in every rule in $P$ every premise 
		$\negf B$ such that $B$ is true (resp. undecided, false) in $I$ by the atoms $0$ (resp. $\frac{1}{2}$, $1$), where $0$ (resp. $\frac{1}{2}$, $1$) are defined to be false (resp. undecided, true) in every interpretation.
It holds that for all 3-valued interpretations $I$ of $P$, $\frac{P}{I}$ is definite (i.e., consists only of definite rules).
This means that $\frac{P}{I}$ has a unique \emph{least} 3-valued interpretation $(T, F)$ with minimal $T$ and maximal $F$ that satisfies all rules.
That is, for all rules $C \leftarrow A_1, \ldots, A_n$, in $\frac{P}{I}$, $C$ is true (resp. \emph{not} false) in $(T, F)$ if for all $i \in \{1, \ldots, n\}$, $A_i$ is true (resp. \emph{not} false) in $(T, F)$.
Given a 3-valued interpretation $I$, the least 3-valued interpretation of $\frac{P}{I}$ is denoted by $\Gamma(I)$.
This leads to the following definition of a \emph{partial stable model} of a logic program, along with the associated notions of consequence.

\begin{definition}\label{defn:psm} \cite{DBLP:journals/fuin/Przymusinski90}
Let $P$ be a logic program.
A 3-valued interpretation $I$ is a \emph{partial stable model} of $P$ iff $I = \Gamma(I)$.
We say that an atom $C$ is a skeptical (resp. credulous) consequence of $P$ iff $C$ is true in all (resp. some) partial stable models of $P$. 
\end{definition}

It has been shown that the above defined notion of skeptical consequence coincides with the \emph{well-founded} semantics~\cite{DBLP:journals/fuin/Przymusinski90}.

\subsection{Logic programming as argumentation}\label{sec:lparg}

Wu et al.~\cite{DBLP:journals/sLogica/WuCG09} have shown that a logic program $P$ can be transformed into an AF $\fw$ in such a way that the consequences of $P$ under the partial stable semantics can be computed by looking at the complete extensions of $\fw$.
The idea is that an argument consists of a conclusion $C \in At_P$, a set of rules $R \subseteq P$ used to derive $C$ and a set $N \subseteq At_P$ of atoms that must be underivable in order for the argument to be acceptable. 
The argument is attacked by another argument with a conclusion $C'$ iff $C' \in N$.
The following definition, apart from notation, is due to Wu et al.~\cite{DBLP:journals/sLogica/WuCG09}.

\begin{definition}
Let $P$ be a logic program.
An instantiated argument is a triple $(C, R, N)$, where $C \in At_P$, $R \subseteq P$ and $N \subseteq At_P$.
We say that $P$ generates $(C, R, N)$ iff either:
\begin{itemize}
\item $r = C \leftarrow \negf B_1, \ldots, \negf B_m$ is a rule in $P$, $R = \{ r \}$ and $N = \{B_1, \ldots, B_m\}$.
\item (1) $r = C \leftarrow A_1, \ldots, A_n, \negf B_1, \ldots, \negf B_m$ is a rule in $P$, (2) $P$ generates, for each $i \in \{1, \ldots, n]$ an argument $(A_i, R_i, N_i)$ such that $r \not \in R_i$, and (3) $R = \{r\} \cup R_1 \cup \ldots \cup R_n$ and $N = \{B_1, \ldots, B_m\} \cup N_1 \cup \ldots \cup N_n$.
\end{itemize}
We denote the set of arguments generated by $P$ by $\args_{P}$.
Furthermore, the attack relation generated by $P$ is denoted by $\attrel_{P}$ and is defined by $(C, R, N) \attrel_{P} (C', R', N')$ iff $C \in N'$.
\end{definition}

The following theorem states that skeptical (resp. credulous) acceptance in $(\args_{P}, \attrel_{P})$ corresponds with skeptical (resp. credulous) consequences in $P$ as defined in definition~\ref{defn:psm}.
It follows from theorems 15 and 16 due to Wu et al.~\cite{DBLP:journals/sLogica/WuCG09}.

\begin{thm}\label{thm:wu}
Let $P$ be a logic program.
An atom $C \in At_{P}$ is a skeptical (resp. credulous) consequence of $P$ iff some $(C, R, N) \in \args_{P}$ is skeptically (resp. credulously) accepted in $(\args_{P}, \attrel_{P})$.
\end{thm}

\subsection{Abduction in logic programming}\label{sec:alp}

The model of abduction in logic programming that we use is based on the model of \emph{extended} abduction studied by Inoue and Sakama~\cite{DBLP:conf/ijcai/InoueS95,DBLP:journals/amai/InoueS99}.
They define an abductive logic program (ALP) to consist of a logic program and a set of atoms called \emph{abducibles}.

\begin{definition}
An abductive logic program is a pair $(P, U)$ where $P$ is a logic program and $U \subseteq At_P$ a set of facts called abducibles.
\end{definition}

Note that, as before, the set $U$ consists of ground facts of the form $C \leftarrow$ (identified with the atom $C$) and can alternatively be regarded as the set of ground instances of a set of non-ground facts.
A hypothesis, according to Inoue and Sakama's model, consists of both a positive element (i.e., abducibles added to $P$) and a negative element (i.e., abducibles removed from $P$).

\begin{definition}\label{defn:hypothesisalp}
Let $\alp = (P, U)$ be an abductive logic program.
A hypothesis is a pair $(\Delta^{+}, \Delta^{-})$ such that $\Delta^{+}, \Delta^{-} \subseteq U$ and $\Delta^{+} \cap \Delta^{-} = \emptyset$.
A hypothesis $(\Delta^{+}, \Delta^{-})$ skeptically (resp. credulously) explains a query $Q \in At_{P}$ if and only if $Q$ is a skeptical (resp. credulous) consequence of $(P \cup \Delta^{+}) \setminus \Delta^{-}$.
\end{definition}

Note that Sakama and Inoue focus on computation of explanations under the stable model semantics of $P$, and require $P$ to be acyclic to ensure that a stable model of $P$ exists and is unique~\cite{DBLP:journals/amai/InoueS99}.
We, however, define explanation in terms of the consequences according to the partial stable models of $P$, which always exist even if $P$ is not acyclic~\cite{DBLP:journals/fuin/Przymusinski90}, so that we do not need this requirement.

The following example demonstrates the previous two definitions.

\begin{example}\label{example:alp}
Let $\alp = (P, U)$ where $P = \{(p \leftarrow \negf s, r), (p \leftarrow \negf s, \negf q), (q \leftarrow \negf p), r\}$ and $U =  \{r, s\}$.
The hypothesis $(\{s\},\emptyset)$ skeptically explains $q$, witnessed by the unique model $I = (\{r, s, q\}, \{p\})$ satisfying $I = \Gamma(I)$.
Similarly, $(\{s\}, \{r\}))$ skeptically explains $q$ and $(\emptyset, \{r\}))$ credulously explains $q$.
\end{example}

\subsection{Instantiated abduction in argumentation}\label{sec:alparg}

In this section we show that an AAF $(\fw, \aafs)$ can be instantiated on the basis of an abductive logic program $(P, U)$.
The idea is that every possible hypothesis $(\Delta^{+}, \Delta^{-})$ maps to an abducible AF generated by the logic program $(P \cup \Delta^{+}) \setminus \Delta^{-}$.
The hypotheses for a query $Q$ then correspond to the abducible AFs that explain the observation $X$ consisting of all arguments with conclusion $Q$.
The construction of $(\fw, \aafs)$ on the basis of $(P, U)$ is defined as follows.

\begin{definition}\label{defn:instantiation}
Let $\alp = (P, U)$ be an abductive logic program.
Given a hypothesis $(\Delta^{+}, \Delta^{-})$, 
	we denote by $\fw_{(\Delta^{+}, \Delta^{-})}$ 
	the AF 
	$	(	\args_{(P \cup \Delta^{+}) \setminus \Delta^{-}}, 
			\attrel_{(P \cup \Delta^{+}) \setminus \Delta^{-}})$.
The AAF \emph{generated by} $\alp$ 
	is denoted by $\paf_{\alp}$ and defined by 
	$\paf_{\alp} = ((\args_{P}, \attrel_{P}), I_\alp)$, 
		where $I_\alp = \{ \fw_{(\Delta^{+}, \Delta^{-})} \mid \Delta^{+}, \Delta^{-} \subseteq U, \Delta^{+} \cap \Delta^{-} = \emptyset \}$.
\end{definition}

The following theorem states the correspondence between the explanations of a query $Q$ in an abductive logic program $\alp$ and the explanations of an observation $X$ in the AAF $\paf_{\alp}$.

\begin{thm}\label{thm:instantiation}
Let $\alp = (P, U)$ be an abductive logic program, $Q \in At_{P}$ a query and $(\Delta^{+}, \Delta^{-})$ a hypothesis.
Let $\paf_{\alp} = (\fw, \aafs)$.
We denote by $X_{Q}$ the set $\{(C, R, N) \in \args_{P} \mid C = Q\}$.
It holds that $(\Delta^{+}, \Delta^{-})$ skeptically (resp. credulously) explains $Q$ iff $\fw_{(\Delta^{+}, \Delta^{-})}$ skeptically (resp. credulously) explains $X_{Q}$.
\end{thm}

\begin{proof}[Proof of theorem~\ref{thm:instantiation}]
Via theorem~\ref{thm:wu} and definitions~\ref{defn:hypothesisalp} and~\ref{defn:instantiation}.
\end{proof}

This theorem shows that our model of abduction in argumentation can indeed be seen as an abstraction of abductive logic programming.

\begin{example}~\label{examp:finalexample}
Let $\alp = (P, U)$ be the ALP as defined in example~\ref{example:alp}. 
All arguments generated by $\alp$ are:
\begin{tabular}{rclrcl}
$a$		& $ = $	&	$( p,	\{(p \leftarrow \negf s, r), r\}				, \{s\}		)$	&	$d$		& $ = $	&	$( r,	\{r\}									, \emptyset)$	\\
$b$		& $ = $	&	$( q,	\{(q \leftarrow \negf p)\}					, \{p\}		)$	&	$e$		& $ = $	&	$( s,	\{s\}									, \emptyset)$\\
$c$		& $ = $	&	$( p,	\{(p \leftarrow \negf s, \negf q)\}				, \{s,q\}	)$	&	&	&	\\
\end{tabular}

\noindent Given these definitions, the AAF in example~\ref{examp:abd} is equivalent to $\paf_{\alp}$.
In example~\ref{example:alp} we saw that $q$ is skeptically explained by $(\{s\},\emptyset)$ and $(\{s\}, \{r\})$, 
	while $(\emptyset, \{r\})$ only credulously explains it.
Indeed, looking again at example~\ref{examp:abd}, we see that $\fwb_1 = F_{(\{s\},\emptyset)}$ and $\fwb_3 = F_{(\{s\}, \{r\})}$ explain skeptical support for the observation $\{b\} = X_{q}$, while $\fwb_2 = F_{(\emptyset, \{r\})}$ only explains credulous support.
\end{example}

\begin{remark}
This method of instantiation shows that, on the abstract level, hypotheses cannot be represented by independently selectable abducible arguments.
The running example shows e.g. that $a$ and $d$ cannot be added or removed independently.
(Cf. remark~\ref{remark1}.)
\end{remark}

\section{Related work}\label{sec:relatedwork}

We already discussed Sakama's~\cite{sakama} model of abduction in argumentation and mentioned some differences.
Our approach is more general because we consider a hypothesis to be a change to the AF that is applied as a whole, instead of a set of independently selectable abducible arguments.
On the other hand, Sakama's method supports a larger range semantics, including (semi-)stable and skeptical preferred semantics.
Furthermore, Sakama also considers observations leading to rejection of arguments, which we do not.

Some of the ideas we applied also appear in work by Wakaki et al.~\cite{DBLP:conf/argmas/WakakiNS09}.
In their model, ALPs generate instantiated AFs and hypotheses yield a division into active/inactive arguments.  

Kontarinis et al.~\cite{kontarinis2013rewriting} use term rewriting logic to compute changes to an abstract AF with the goal of changing the status of an argument.
Two similarities to our work are:
(1) our production rules to generate dialogues can be seen as a kind of term rewriting rules. 
(2) their approach amounts to rewriting goals into statements to the effect that certain attacks in the AF are enabled or disabled.
These statements resemble the moves $\hpmoveattack{x}{y}$ and $\hpmoveattackneg{x}{y}$ in our system.
However, they treat attacks as entities that can be enabled or disabled independently.
As discussed, different arguments (or in this case attacks associated with arguments) cannot be regarded as independent entities, if the abstract model is instantiated.

Goal oriented change of AFs is also studied by Baumann~\cite{DBLP:journals/ai/Baumann12}, Baumann and Brewka~\cite{DBLP:conf/comma/BaumannB10}, Bisquert et al.~\cite{DBLP:conf/sum/BisquertCSL13} and Perotti et al.~\cite{DBLP:conf/tafa/BoellaGPTV11}.
Furthermore, Booth et al.~\cite{DBLP:conf/sum/BoothKRT13} and Coste-Marquis et al.~\cite{coste2013revision} frame it as a problem of \emph{belief revision}.
Other studies in which changes to AFs are considered include~\cite{DBLP:conf/argmas/BoellaKT09,cayrol2010change,DBLP:journals/ai/LiaoJK11,oikarinen2011characterizing}.

\section{Conclusions and Future work}\label{sec:conclusion}
We developed a model of abduction in abstract argumentation, in which changes to an AF act as explanations for skeptical/credulous support for observations.
We presented sound and complete dialogical proof procedures for the main decision problems, i.e., finding explanations for skeptical/credulous support.
In addition, we showed that our model of abduction in abstract argumentation can be seen as an abstract form of abduction in logic programming.

As a possible direction for future work, we consider the incorporation of additional criteria for the selection of good explanations, such as minimality with respect to the added and removed arguments/attacks, as well as the use of arbitrary preferences over different abducible AFs.
An interesting question is whether the proof theory can be adapted so as to yield only the preferred explanations.

\section{Acknowledgements}

Richard Booth is supported by the Fonds National de la Recherche, Luxembourg (DYNGBaT project).

\bibliography{ECAI-156}

\begin{thebibliography}{10}

\bibitem{DBLP:conf/sum/2013}
{\em Scalable Uncertainty Management - 7th International Conference, SUM 2013,
  Washington, DC, USA, September 16-18, 2013. Proceedings}, 2013.

\bibitem{DBLP:journals/ai/Baumann12}
Ringo Baumann, `Normal and strong expansion equivalence for argumentation
  frameworks', {\em Artif. Intell.}, {\bf 193},  18--44, (2012).

\bibitem{DBLP:conf/comma/BaumannB10}
Ringo Baumann and Gerhard Brewka, `Expanding argumentation frameworks:
  Enforcing and monotonicity results', in {\em Proc. COMMA}, pp. 75--86,
  (2010).

\bibitem{DBLP:conf/sum/BisquertCSL13}
Pierre Bisquert, Claudette Cayrol, Florence~Dupin de~Saint-Cyr, and
  Marie-Christine Lagasquie-Schiex, `Enforcement in argumentation is a kind of
  update', In {\em SUM\/} \cite{DBLP:conf/sum/2013}, pp. 30--43.

\bibitem{DBLP:conf/tafa/BoellaGPTV11}
Guido Boella, Dov~M. Gabbay, Alan Perotti, Leon van~der Torre, and Serena
  Villata, `Conditional labelling for abstract argumentation', in {\em TAFA},
  pp. 232--248, (2011).

\bibitem{DBLP:conf/argmas/BoellaKT09}
Guido Boella, Souhila Kaci, and Leendert van~der Torre, `Dynamics in
  argumentation with single extensions: Attack refinement and the grounded
  extension (extended version)', in {\em ArgMAS}, pp. 150--159, (2009).

\bibitem{DBLP:conf/aldt/BoothKR13}
Richard Booth, Souhila Kaci, and Tjitze Rienstra, `Property-based preferences
  in abstract argumentation', in {\em ADT}, pp. 86--100, (2013).

\bibitem{DBLP:conf/sum/BoothKRT13}
Richard Booth, Souhila Kaci, Tjitze Rienstra, and Leon van~der Torre, `A
  logical theory about dynamics in abstract argumentation', In {\em SUM\/}
  \cite{DBLP:conf/sum/2013}, pp. 148--161.

\bibitem{caminada2010preferred}
Martin Caminada, `Preferred semantics as socratic discussion', in {\em
  Proceedings of the 11th AI* IA Symposium on Artificial Intelligence}, pp.
  209--216, (2010).

\bibitem{cayrol2010change}
Claudette Cayrol, Florence Dupin~de Saint-Cyr, and Marie-Christine
  Lagasquie-Schiex, `Change in abstract argumentation frameworks: Adding an
  argument', {\em Journal of Artificial Intelligence Research}, {\bf 38}(1),
  49--84, (2010).

\bibitem{coste2013revision}
Sylvie Coste-Marquis, S{\'e}bastien Konieczny, Jean-Guy Mailly, and Pierre
  Marquis, `On the revision of argumentation systems: Minimal change of
  arguments status', {\em Proc. TAFA}, (2013).

\bibitem{DBLP:journals/ai/Dung95}
Phan~Minh Dung, `On the acceptability of arguments and its fundamental role in
  nonmonotonic reasoning, logic programming and n-person games', {\em Artif.
  Intell.}, {\bf 77}(2),  321--358, (1995).

\bibitem{DBLP:conf/ijcai/InoueS95}
Katsumi Inoue and Chiaki Sakama, `Abductive framework for nonmonotonic theory
  change', in {\em IJCAI}, pp. 204--210. Morgan Kaufmann, (1995).

\bibitem{DBLP:journals/amai/InoueS99}
Katsumi Inoue and Chiaki Sakama, `Computing extended abduction through
  transaction programs', {\em Ann. Math. Artif. Intell.}, {\bf 25}(3-4),
  339--367, (1999).

\bibitem{kontarinis2013rewriting}
Dionysios Kontarinis, Elise Bonzon, Nicolas Maudet, Alan Perotti, Leon van~der
  Torre, and Serena Villata, `Rewriting rules for the computation of
  goal-oriented changes in an argumentation system', in {\em Computational
  Logic in Multi-Agent Systems},  51--68, Springer, (2013).

\bibitem{DBLP:journals/ai/LiaoJK11}
Beishui Liao, Li~Jin, and Robert~C. Koons, `Dynamics of argumentation systems:
  A division-based method', {\em Artif. Intell.}, {\bf 175}(11),  1790--1814,
  (2011).

\bibitem{Modgil2009}
Sanjay Modgil and Martin Caminada, `Proof theories and algorithms for abstract
  argumentation frameworks', in {\em Argumentation in Artificial Intelligence},
   105--129, (2009).

\bibitem{oikarinen2011characterizing}
Emilia Oikarinen and Stefan Woltran, `Characterizing strong equivalence for
  argumentation frameworks', {\em Artificial intelligence}, {\bf 175}(14-15),
  1985--2009, (2011).

\bibitem{DBLP:journals/fuin/Przymusinski90}
Teodor~C. Przymusinski, `The well-founded semantics coincides with the
  three-valued stable semantics', {\em Fundam. Inform.}, {\bf 13}(4),
  445--463, (1990).

\bibitem{sakama}
Chiaki Sakama, `Abduction in argumentation frameworks and its use in debate
  games', in {\em Proceedings of the 1st International Workshop on Argument for
  Agreement and Assurance (AAA)}, (2013).

\bibitem{DBLP:conf/argmas/WakakiNS09}
Toshiko Wakaki, Katsumi Nitta, and Hajime Sawamura, `Computing abductive
  argumentation in answer set programming', in {\em Proc. ArgMAS}, pp.
  195--215, (2009).

\bibitem{DBLP:journals/sLogica/WuCG09}
Yining Wu, Martin Caminada, and Dov~M. Gabbay, `Complete extensions in
  argumentation coincide with 3-valued stable models in logic programming',
  {\em Studia Logica}, {\bf 93}(2-3),  383--403, (2009).

\end{thebibliography}

\end{document}